%% file: main.tex
\pdfoutput=1
\documentclass{article}
\PassOptionsToPackage{numbers, compress}{natbib}

    \usepackage[preprint]{neurips_2023}

\usepackage[utf8]{inputenc} 
\usepackage[T1]{fontenc}    
\usepackage{hyperref}       
\usepackage{url}            
\usepackage{booktabs}       
\usepackage{amsfonts}       
\usepackage{nicefrac}       
\usepackage{microtype}      
\usepackage[table]{xcolor}         

\usepackage{subfigure}
\usepackage{booktabs} 
\usepackage{xspace}
\usepackage{enumitem}

\usepackage{algorithm}
\usepackage{algorithmic}

\usepackage{amsmath}
\usepackage{amssymb}
\usepackage{mathtools}
\usepackage{amsthm}
\usepackage{stmaryrd}
\usepackage{wrapfig}
\usepackage{makecell}
\usepackage{colortbl}
\usepackage{nicematrix}

\usepackage{comment}

\usepackage[capitalize,noabbrev]{cleveref}

\theoremstyle{plain}
\newtheorem{theorem}{Theorem}

\theoremstyle{definition}

\newtheorem{assumption}[theorem]{Assumption}
\theoremstyle{remark}

\newtheoremstyle{TheoremNum}
    {\topsep}{\topsep}
    {\itshape}
    {}
    {\bfseries}
    {.}
    { }
    {\thmname{#1}\thmnote{ \bfseries #3}}
\theoremstyle{TheoremNum}
\newtheorem{reptheorem}{Theorem}

\usepackage[textsize=tiny]{todonotes}

\input{math_commands.tex}

\input{macros.tex}
\usepackage{caption}
\usepackage{pifont}
\usepackage{makecell}
\usepackage{multirow}
\usepackage{soul}

\title{Certifiably Robust Reinforcement Learning through Model-Based Abstract Interpretation}

\author{Chenxi Yang \\
    UT Austin \\
  \And
  Greg Anderson \\
  UT Austin \\
  \AND
  Swarat Chaudhuri \\
  UT Austin \\
}

\begin{document}

\maketitle

\newcommand{\red}[1]{\textcolor{black}{#1}} 
\newcommand{\blue}[1]{\textcolor{black}{#1}}
\newcommand{\sys}{\textsc{Carol}\xspace} %

\begin{abstract}
We present a reinforcement learning (RL) framework in which the learned policy comes with a machine-checkable \emph{certificate of provable adversarial robustness}. 
Our approach, called \sys, learns a model of the environment. In each learning iteration, it uses the current version of this model and an external \emph{abstract interpreter} to construct a differentiable signal for provable robustness. This signal is used to guide learning, and the abstract interpretation used to construct it directly leads to the robustness certificate returned at convergence. We give a theoretical analysis that bounds the worst-case accumulative reward of \sys. We also experimentally evaluate \sys on four MuJoCo environments with continuous state and action spaces. On these tasks, \sys learns policies that, when contrasted with policies from two state-of-the-art robust RL algorithms, exhibit: (i) markedly enhanced certified performance lower bounds; and (ii) comparable performance under empirical adversarial attacks. 
\end{abstract}

\section{Introduction}\label{sec:intro}
\input{sec_intro}


\section{Background}\label{sec:bg}
\input{sec_bg}

\section{Problem Formulation}\label{sec:prob}
\input{sec_problem}

\section{Learning Algorithm}\label{sec:alg}
\input{sec_algo}

\section{Theoretical Analysis}\label{sec:the}
\input{sec_theory}

\section{Evaluation}\label{sec:eval}
\input{sec_evaluation}

\section{Related Work}\label{sec:related}
\input{sec_related}

\section{Discussion} \label{sec:conclusion}
\input{sec_conclusion}





\bibliographystyle{plainnat}
\bibliography{ref}


\newpage
\appendix
\onecolumn
\input{Appendix/symbols}
\input{Appendix/proofs}
\input{Appendix/implementation}
\input{Appendix/examples}
\input{Appendix/bound}
\input{Appendix/discussion}

\end{document}

%% file: math_commands.tex
\usepackage{amsmath,amsfonts,bm}

\def\eqref#1{equation~\ref{#1}}

\def\mI{{\bm{I}}}

\def\mSigma{{\bm{\Sigma}}}

\DeclareMathAlphabet{\mathsfit}{\encodingdefault}{\sfdefault}{m}{sl}
\SetMathAlphabet{\mathsfit}{bold}{\encodingdefault}{\sfdefault}{bx}{n}

\newcommand{\E}{\mathbb{E}}

\DeclareMathOperator*{\argmax}{arg\,max}
\DeclareMathOperator*{\argmin}{arg\,min}

%% file: macros.tex
\newcommand{\states}{\mathcal{S}}
\newcommand{\actions}{\mathcal{A}}
\newcommand{\adv}{\nu}
\newcommand{\advset}{\mathbb{A}}
\newcommand{\mdp}{\mathcal{M}}

\newcommand{\conc}{\beta}

\newcommand{\expec}{\mathop{\mathbb{E}}}
\newcommand{\var}{\mathrm{Var}}
\newcommand{\prob}{\mathrm{Pr}}

\newcommand{\norm}{\mathcal{N}}
\newcommand{\mZero}{{\bm{0}}}
\newcommand{\Reals}{\mathbb{R}}

\newcommand{\policyLoss}{L^{\text{normal}}}
\newcommand{\safetyLoss}{L^{\text{symbolic}}}

\newcommand{\wcar}{\textsc{Wcar}\xspace}

%% file: sec_intro.tex
Reinforcement learning (RL) is an established approach to control tasks \cite{polydoros2017survey, mnih2015human}, including safety-critical ones \cite{cheng2019end, sallab2017deep}. However, state-of-the-art RL methods 
use neural networks as policy representations. This makes them vulnerable to adversarial attacks in which carefully crafted perturbations to a policy's inputs cause it to behave incorrectly.  
These problems are even more severe in RL than in supervised learning, as the effects of successive mistakes can cascade over a long time horizon. 

These challenges have motivated research on RL algorithms that are robust to adversarial perturbations. In general, adversarial learning techniques can be divided into best-effort heuristic defenses and \emph{certified} approaches that guarantee provable robustness.
The latter are preferable as heuristic defenses are often defeated by counterattacks \cite{russo2019optimal}. While many certified defenses are known for the supervised learning setting \cite{mirman2018differentiable, cohen2019certified,wong2018provable},  
extending these methods to RL has been difficult. The reason is that RL involves a black-box environment. 
{To ensure the \red{certified} robustness of an RL policy, one needs to reason about repeated interactions between the policy, the environment, and the adversary,} and there is no general approach to doing so. Existing approaches to deep certified RL typically sidestep the challenge through various simplifying assumptions, for example, that the perturbations are stochastic rather than adversarial \cite{kumar2021policy}, that the certificate only applies to one-shot interactions between the policy and the environment \cite{oikarinen2021robust,zhang2020robust}, or that the action space is discrete \cite{lutjens2020certified}. 

In this paper, we develop a framework, called \textbf{\sys} (\textbf{C}ertifi\textbf{A}bly \textbf{RO}bust Reinforcement \textbf{L}earning), 
that fills this gap in the literature. We reason about adversarial dynamics over entire episodes by learning a \emph{model} of the environment and repeatedly composing it with the policy and the adversary. 
To this end, we consider a {state-adversarial} Markov Decision Process \cite{zhang2020robust} in which the observed states are adversarially attacked states of the original environment. \red{This threat model aligns with many existing efforts on robust RL \cite{oikarinen2021robust, zhang2020robust, lin2017tactics, lutjens2020certified, fischer2019online} and is also important for real-world RL agents under unpredictable sensor noise.} During exploration, our algorithm learns a model of the environment using an existing model-based reinforcement learning algorithm \cite{janner2019trust}. We perform \emph{abstract interpretation} \cite{cousot1977abstract,mirman2018differentiable} over compositions of the current policy and the learned environment model to estimate worst-case bounds on the agent's adversarial reward. 
The lower bound on the reward is then used to guide the learning.

A key benefit of our \red{\emph{model-based abstract interpretation}} approach is that it not only computes bounds on a policy's worst-case reward but also offers a \emph{proof} of this fact if it holds. A {certificate of robustness} in our framework consists of such a proof.

Our results include a theoretical analysis of our learning algorithm, which shows that our learned certificates give probabilistically sound lower bounds on the accumulative reward of any allowed adversary. We also empirically evaluate \sys over four high-dimensional MuJoCo environments (Hopper, Walker2d, Halfcheetah, and Ant). We demonstrate that {\sys} is able to successfully learn certified policies for these environments and that our strong certification requirements do not compromise empirical performance.
To summarize, our main contributions are as follows:
\begin{itemize}[leftmargin=10pt,itemsep=0pt,topsep=0pt]
     \item We offer \sys, the first 
     RL framework to guarantee episode-level certifiable adversarial robustness in the presence of continuous states and actions. The framework is based on a new combination of model-based learning and abstract interpretation that can be of independent interest.
    \item We give a rigorous theoretical analysis that establishes the (probabilistic) soundness of \sys.
    \item We give experiments on four MuJoCo domains that establish \sys as a new state-of-the-art for certifiably robust RL. 
\end{itemize}

%% file: sec_bg.tex
\label{prob:mdp}

\textbf{Markov Decision Processes (MDPs).} We start with the standard definition of an \emph{Markov Decision Process} (MDP) $\mdp = (\states, \actions, r, P, \states_0)$. Here, $\states$ is a set of states and $\actions$ is a set of actions; \red{for simplicity of presentation, we assume these sets to be $\Reals^k$ and $\Reals^m$ for suitable dimensionality $k$ and $m$}.  
$\states_0 $ is a distribution of initial states; $P(s' \mid s, a)$, for $s,s' \in \states$ and $a \in \actions$, is a probabilistic transition function; $r(s, a)$ for $s\in \states, a \in  \actions$ is a real-valued reward function. Our method assumes an additional property that is commonly satisfied in practice: that $P(s' \mid s, a)$ has the form $\mu_P(s, a) + f_P(s')$, where $f_P(s')$ is a distribution independent of $(s, a)$ and $\mu_P$ is deterministic. 

A \emph{policy} in $\mdp$ is a distribution $\pi(a \mid s)$ with $s \in \states$ and $a \in \actions$. A (finite) \emph{trajectory} $\tau$ is a sequence  $s_0, a_0, s_1, a_1, \ldots$ such that $s_0 \sim \states_0$, each $a_i \sim \pi(s_i)$, and each $s_{i+1} \sim P(s' \mid s_i, a_i)$. We denote by $R(\tau) = \sum_i r(s_i, a_i)$ the aggregate (undiscounted) reward along a trajectory $\tau$, and by $R(\pi)$ the expected reward of trajectories unrolled under $\pi$.  
 
\textbf{State-Adversarial MDPs.} We model adversarial dynamics using 
\emph{state-adversarial MDPs} \cite{zhang2020robust}. 
Such a structure is a pair $\mdp_\adv = (\mdp, B)$, where $\mdp = (\states, \actions, r, P, \states_0)$ is an MDP, 
\red{and $B : \states \to \mathcal{P}(\states)$ is a \emph{perturbation map}, where  $\mathcal{P}(\states)$ is the power set of $\states$. Intuitively, $B(s)$ is 
the set of all states that can result from adversarial perturbations of $s$}. 

Suppose we have a policy $\pi$ in the underlying MDP $\mdp$.
In an attack scenario, 
an adversary $\adv$ perturbs the \emph{observations} of the agent at a state $s$. 
As a result, rather than choosing an action from $\pi(a \mid s)$, the agent now chooses an action from $\pi(a \mid \adv(s))$. However, the environment transition is still sampled from $P(s' \mid s, a)$ and \emph{not} $P(s' \mid \adv(s), a)$, as the ground-truth state does not change under the attack. 
We denote by $\pi \circ \adv$ the state-action mapping that results when $\pi$ is used under this attack scenario.

Naturally, if $\adv$ can arbitrarily perturb states, then adversarially robust learning is intractable. Consequently, we constrain $\adv$ using $B$, requiring $\adv(s) \in B(s)$ for all $s \in \states$.  
We denote the set of allowable adversaries in $\mdp_\adv$ as $\advset_B = \left \{ \adv : \states \to \states \mid \forall s \in \states.\  \adv(s) \in B(s) \right \}$. 

\textbf{Abstract Interpretation.} We certify adversarial robustness using 
\emph{abstract interpretation} \cite{cousot1977abstract}, a classic framework for worst-case safety analysis of systems. 
Here, one represents sets of values --- e.g., system states, actions, and reward values --- using symbolic representations (\emph{abstractions}) in a predefined language.
For example, we can set our \emph{abstract states} to be
hyperintervals that maintain upper and lower bounds in each state space dimension. 
We denote abstract values with the superscript $\#$. For a set of concrete states, $S$, $\alpha(S)$ denotes the smallest abstract state which contains $S$. For an abstract state $s^\#$,  $\conc(s^\#)$ is the set of concrete states represented by $s^\#$. {For simplicity of notation, we assume that functions over concrete states are lifted to their abstract analogs.} For example, $\pi(s^\#)$ is short-hand for a function $\pi^\#(s^\#) = \alpha(\{\pi(s): s \in \beta(s^\#)\})$. Similar functions are defined for abstract reward values, actions, and so on. 

The core of abstract interpretation is the propagation of abstract states $s^\#$ through a function $f(s)$ that captures single-step system dynamics. 
For propagation, we assume that we have access to a map $f^\#(s^\#)$ that ``lifts" $f$ to abstract states. This function must satisfy the property
$\conc(f^\#(s^\#)) \supseteq \{f(s):  s \in \conc(s^\#)\}.$
Intuitively, $f^\#$ \emph{overapproximates} the behavior of $f$: while the abstract state $f^\#(s^\#)$ may include some states that are not actually reachable through the application of $f$ to states encoded by $s^\#$, it will \emph{at least} include \emph{every} state that is reachable this way. 

By starting with an abstraction $s^\#_0$ of the
initial states and using abstract interpretation to propagate this abstract state through the transition function $f$, we can obtain an abstract state $s^\#_i$ which includes all states of the system that are reachable in $i$ steps for increasing $i$. The sequence of abstract states $\tau^\#=s^\#_0s^\#_1s^\#_2\dots$ is called 
an \emph{abstract trace}.

%% file: sec_problem.tex
\begin{table}[t]
\centering
\scalebox{0.68}{
\begin{NiceTabular}{c|ccccccccc}
\toprule
$s_{0} = 1.0$ & \makecell{$\adv(s)$ = $\langle s_{0} + \epsilon_{0}$, \\ $s_{1} + \epsilon_{1} \rangle$} & $s_0$ & $s_{\text{obs}0}$ & $a_{0}$ & $s_{1}$ & $s_{\text{obs}1}$ & $a_{1}$ & $R(\pi \circ \adv)$ & $\expec_{\tau \sim \pi \circ \adv}[R]$ \\ \midrule
No-Adv & $\epsilon_{0} = \epsilon_{1} = 0.0$ & 1 & 1 & 1 & $2 + e$ & $2 + e$ & $2 + e$ & $6 + 2e$ & $6$\\
\rowcolor[gray]{.8} Adv-1 & \makecell{$\epsilon_{0} = 0.1$, \\ $\epsilon_{1} = -0.4$} & 1 & 1.1 & 1.1 & $2.1 + e$ & $1.7 + e$ & $1.7 + e$ &  $5.9 + 2e$ & $5.9$\\
Adv-2 & \makecell{$\epsilon_{0} =-0.2$, \\ $\epsilon_{1} = -0.3$} & 1 & 0.8 & 0.8 & $1.8 + e$ & $1.5 + e$ & $1.5 + e$ & $5.1 + 2e$ & $5.1$\\
\rowcolor[gray]{.8}  Reward Bound ($R^\#$) & \makecell{ $\epsilon_{t} \in [-0.5, 0.5]$, \\ $\epsilon^{\#}_{t} = [-0.5, 0.5]$} & 1 & $1 + [-0.5, 0.5]$ & $[0.5, 1.5]$ & $[1.5, 2.5] + e$ & $[1, 3] + e$ & \makecell{$[1 + e$, \\ $3 + e]$} & \makecell{$[4 + 2e$, \\ $8 + 2e]$} & $[4, 8]$ \\
\bottomrule
\end{NiceTabular}
}
\vspace{5pt}
\caption{Example of reward bound calculation. The MDP in this example has initial state set $S_0=[1.0, 1.0]$, white-box transition function $P(s'|s, a)=s + a + \mathcal{N}(0, 1)$, reward function $r(s, a) = s + a$, and adversary $\adv(s) \in [s-0.5, s+0.5]$. $\epsilon_t$ denotes the disturbance added on step $t$. $e$ represents the stochasticity from the transition, where $e \sim \mathcal{N}(0, 1)$. We aim to certify over the worst-case accumulative reward of a deterministic policy $\pi$ defined as $\pi(s) = s$. We define the \textit{worst-case} by considering all potential adversaries while still considering the expected behavior over the stochastic environment, $P$. As shown in the above table, we first demonstrate three traces from fixed adversaries. In the last row, we demonstrate how we consider all the adversary behaviors through an abstract trace via abstract interpretation with intervals. The worst-case accumulative reward in this example is $4$ as $\expec_{\mathcal{N}} [e] = 0$. The abstract trace over all the adversaries in the last row is our certificate which serves as a proof that the policy satisfies our property. We want to ensure that the lower bound of the $R^\#$ should not be lower than a threshold. In training, we use the abstract trace to compute a loss to guide the learning process.}
\label{tab:example}
\vspace{-10pt}
\end{table}

We start by defining robustness. Assume an adversarial MDP $\mdp_\adv$, a
policy $\pi$, and a {threshold} $\Delta > 0$. A \emph{robustness property} is a constraint $\phi(\pi, \Delta)$ of the form 
$\forall \adv \in \advset_B.\  R(\pi) - R(\pi \circ \adv) < \Delta.$
Intuitively, $\phi$ states that no allowable adversary can reduce the expected reward of $\pi$ by more than $\Delta$. 

Our goal in this paper is to learn policies that are \emph{provably} robust. 
Accordingly, we expect our learning algorithm to produce, in addition to a policy $\pi$, a \emph{certificate}, or proof, $c$ of robustness. 
Formally, let $\Pi$ be the universe of all policies in a given state-adversarial MDP $\mdp_\adv = (\mdp, B)$. 
For a policy $\pi$ and a robustness property $\phi$, we write $\pi \vdash_c \phi$ if $\pi$ \emph{provably} satisfies $\phi$, and $c$ is a proof of this fact. 

The problem of \emph{reinforcement learning with robustness certificates} is now defined as:
\begin{equation}\label{eq:goal}
(\pi^*, c) = \argmax_{\pi \in \Pi} \expec_{\tau \sim (\mdp, \pi)} \left[ R(\tau)\right], 
\textrm{s.t.~~} \pi^* \vdash_c \phi.
\end{equation}
That is, we want to find a policy that maximizes the standard expected reward in RL but also ensures that the expected worst-case \emph{adversarial }reward is provably above a threshold. 

Our certificates can be constructed using a variety of symbolic or statistical techniques. In \sys, certificates are constructed using an abstract interpreter.
Suppose we have a policy $\pi$ and an abstract trace $\tau^\#=s^\#_0s^\#_1\dots s^\#_n$ such that for all length-$n$ trajectories $\tau = s_0\dots s_n$ and all $i$, $s_i \in \beta(s^\#_i)$. The abstract trace allows us to compute a lower bound on the expected reward for $\pi$ and also serves as a proof of this bound. We give an example of such certification in a simple state-adversarial MDP, assumed to be available in white-box form, in \Cref{tab:example}.

A challenge here is that abstract interpretation requires a white-box transition function, which is not available in RL. We overcome this challenge by learning a model of the environment during exploration. Model learning is a source of error, so our certificates are probabilistically sound, i.e., they guarantee robustness with high probability. However, this error only depends on the underlying model-based RL algorithm and does not restrict the adversary.

\vspace{-5pt}

%% file: sec_algo.tex
\begin{wrapfigure}{r}{0.5\linewidth}
    \vspace{-15pt} 
    \centering
    \includegraphics[scale=0.6,bb=0 0 300 150]{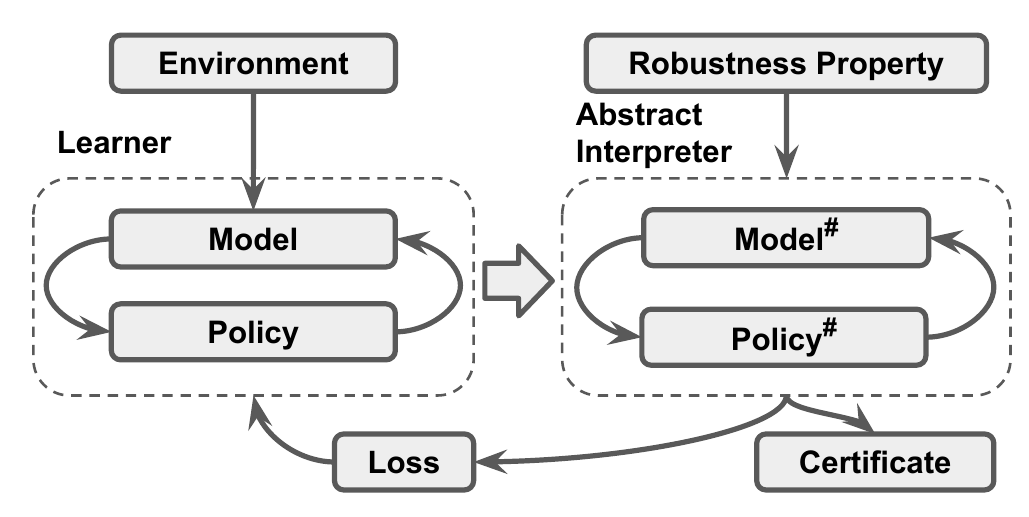}
    \caption{Schematic of \sys}
    \label{fig:algo-overview}
    \vspace{-15pt}
\end{wrapfigure}

Now we present the \sys framework. The framework (\Cref{fig:algo-overview}) has two key components: a model-based learner and an abstract interpreter. During each training round, the learner maintains a model of the environment dynamics and a policy. These are sent to the abstract interpreter, which calculates a lower bound on the abstract reward. The lower bound is used to compute a differentiable loss the learner uses in the next iteration of learning. At convergence, the abstract trace computed during abstract interpretation is returned as a certificate of robustness.

\textbf{Abstract Interpretation in \sys.} 
Now we describe the abstract interpreter in \sys in more detail. 
Recall that our definition of robustness compares the \emph{expected} reward of the original policy to the \emph{expected} reward of the policy under an adversarial perturbation. As a result, our verifier is designed to reason about the worst-case reward under adversarial perturbations, while considering average-case behavior for stochastic policies and environments. \Cref{alg:proof} finds a lower bound on this worst-case expected reward using abstract interpretation to overapproximate the adversary's possible behaviors along with sampling to approximate the average-case behavior of the policy and environment. We denote this lower bound from \Cref{alg:proof} as \textit{worst-case accumulative reward} (\textbf{\wcar}), which is also used to measure the certified performance in our evaluation.

In more detail, Algorithm~\ref{alg:proof} proceeds by sampling a starting state $s_0 \sim \states_0$. Then in Algorithm~\ref{alg:abs-rollout} for each time step, we find an overapproximation $s_{\text{obs}_i}^\#$ which includes all of the possible ways the adversary may perturb $s_i$. Based on this approximation, we \emph{sample} a new approximation from the policy $\pi$. Intuitively, this may be done by using a policy $\pi$ whose randomness does \emph{not} depend on the current state of the system. More formally, $\pi(a \mid s) = \mu_\pi(s) + f_{\pi}(a)$ where $f_{\pi}(a)$ is a distribution with zero mean which is independent of $s$. Then $a_i^\#$ may be computed as $\mu_\pi(s_{\text{obs}_i}^\#) + \alpha(\{e\})$ where $e \sim f_{\pi}(a)$. Once the abstract action is computed, we may find the new (abstract) state and reward using the environment model $E$. \red{The model is assumed to satisfy a PAC-style bound, i.e., there exist $\delta_E$ and $\varepsilon_E$ such that with probability at least $1 - \delta_E$, $\| E(s, a) - P(s, a) \| \le \varepsilon_E$. The values of $\delta_E$ and $\varepsilon_E$ can be measured during model construction. }

One way to understand Algorithm~\ref{alg:proof} is to consider pairs of abstract and concrete trajectories in which the randomness is resolved in the same way. Specifically, if $\pi(a \mid s) = \mu_\pi(s) + f_{\pi}(a)$ and $E(s' \mid s, a) = \mu(s, a) + f_{E}(s')$, the initial state $s_0$ combined with the sequence of values $e_i \sim f_{\pi}(a)$ and $e'_i \sim f_{E}(s')$ for $0 \le i \le T$ uniquely determine a trajectory. For a given set of values, the reward bound $\inf \conc(R_{\text{min}_t}^\#)$ represents the worst-case reward under any adversary \emph{for a particular resolution of the randomness} in the environment and the policy. The outer loop of Algorithm~\ref{alg:proof} approximates the expectation over these different random values by sampling. Theorem~\ref{thm:sound} in Section~\ref{sec:the} shows formally that with high probability, Algorithm~\ref{alg:proof} gives a lower bound on the true adversarial reward. 

\begin{algorithm*}[t]
\caption{Worst-Case Accumulative Reward (\wcar)}
\label{alg:proof}
\begin{algorithmic}[1]
\small{
    \STATE \textbf{Input:} policy $\pi$, model $E$ \\
    \STATE \textbf{Output:} worst case reward of $\pi$ under any adversary
    \FOR{$t$ from 1 to $N$}
        \STATE Sample an initial state $s_0 \sim \states_0$
        \STATE Get the worst case reward $R_{\text{min}_t}$ using Algorithm~\ref{alg:abs-rollout} over horizon $T$ starting from $s_0$
    \ENDFOR
    \STATE \textbf{return} $\frac{1}{N}\sum_{t=1}^N R_{\text{min}_t}$
}
\end{algorithmic}
\end{algorithm*}
\begin{algorithm*}[t]
\caption{Worst-case rollout under adversarial perturbation}
\label{alg:abs-rollout}
\begin{algorithmic}[1]
\small{
    \STATE \textbf{Input:} Initial state $s_0$, rollout horizon $T$
    \STATE \textbf{Output:} Worst-case reward of $\pi$ starting from $s_0$ over one random trajectory
    \STATE Abstract the initial state and reward: ${s_\text{original}}_0^\# \leftarrow \alpha(\{s_0\}), \quad {R_{\text{min}_t}^{\#}}_{i} \leftarrow \alpha \left ( \{ 0 \} \right)$
    \FOR{$i$ from 1 to $T$}
        \STATE Abstract over possible perturbations: ${s_\text{obs}^{\#}}_{i} \leftarrow  B({s_\text{original}^{\#}}_{i}) $ \label{line:p3}
        \STATE Calculate symbolic predicted actions: $a_i^\# \leftarrow \pi({s_{\text{obs}}^{\#}}_i) $ \label{line:p1}
        \STATE Calculate symbolic next-step states and rewards: \\ ${s_\text{original}^{\#}}_{i+1}, r^{\#}_{i} \leftarrow  E_{\theta}({s_\text{original}^{\#}}_{i}, a^{\#}_i) + \alpha(\{ x \mid \| x \| \le \varepsilon_E \}) $ \label{line:p2}
        \STATE Update worst-case reward: ${R_{\text{min}}^{\#}}_{t} \leftarrow  {R_{\text{min}}^{\#}}_{t} + r^{\#}_{i} $ \label{line:p4}
    \ENDFOR
    \STATE \textbf{return} $\inf \conc({R_{\text{min}}^{\#}}_{t})$
}
\end{algorithmic}
\end{algorithm*}
\begin{algorithm*}[ht]
\caption{Certifiably Robust Reinforcement Learning}
\label{alg:main}
\begin{algorithmic}[1]
\small{
    \STATE Initialize a random policy $\pi_\psi$, random environment model $E_\theta$, and empty model dataset $\mathcal{D}_\text{model}$.
    \STATE Initialize an environment dataset $\mathcal{D}_\text{env}$ by unrolling trajectories under a random policy.
    \FOR{$N$ epochs}
        \STATE Train model $E_\theta$ on $\mathcal{D}_\text{env}$ via maximum likelihood
        \STATE Unroll $M$ trajectories int he model under $\pi_\psi$; add to $\mathcal{D_\text{model}}$
        \STATE Take action in environment according to $\pi_\psi$; add to $\mathcal{D_\text{env}}$
        \FOR{$G$ gradient updates}
            \STATE Calculate normal policy loss $\policyLoss(\pi_\psi, \mathcal{D_\text{model}})$ as in MBPO \cite{janner2019trust}
            \STATE Sample $\langle s_t, a_t, s_{t+1}, r_t\rangle$ uniformly from $\mathcal{D}_\text{model}$
            \STATE Rollout $\pi$ starting from $s_t$ under $E_\theta$ for $T_\text{train}$ steps and compute the total reward $R^o$
            \STATE Compute the worst-case reward $R_\text{min}$ using Algorithm~\ref{alg:abs-rollout} over horizon $T_\text{train}$.
            \STATE Compute the robustness loss $\safetyLoss(\pi_\psi, E_\theta) \gets R^o - R_\text{min}$
            \STATE Update policy parameters: $\psi \gets \psi - \alpha \nabla_\psi (\policyLoss(\pi_\psi, \mathcal{D}_\text{model}) + \lambda ( \safetyLoss(\pi_\psi, E_\theta) - \Delta))$
            \STATE Update Lagrange multiplier: $\lambda \gets \max(0, \lambda + \alpha' (\safetyLoss(\pi_\psi, E_\theta) - \Delta))$
        \ENDFOR
        \STATE Unroll $n$ trajectories in the true environment under $\pi_\psi$; add to $\mathcal{D}_{\text{env}}$
    \ENDFOR
}
\end{algorithmic}
\end{algorithm*}

\textbf{Learning in \sys.}
Now we discuss how to learn a policy and environment model which may be proven robust by Algorithm~\ref{alg:proof}. At a high level, Algorithm~\ref{alg:main} works by introducing a symbolic loss term $\safetyLoss_\psi$ which measures the robustness of the policy. Because robustness is a constrained optimization problem, we use this symbolic loss with a Lagrange multiplier in an alternating gradient descent scheme to find the optimal robust policy. Formally, for a given environment model $E$, the inner loop in Algorithm~\ref{alg:main} solves the optimization problem
\[ \argmin_\psi \policyLoss(\pi_\psi, \mathcal{D}_\text{model}) \quad \text{s.t.} \quad \safetyLoss(\pi, E) \le \Delta \]
via the Lagrangian
$\argmin_\psi \max_{\lambda \ge 0} \policyLoss(\pi_\psi, \mathcal{D}_\text{model}) + \lambda (\safetyLoss(\pi, E) - \Delta).$ 
We ensure that solving this problem solves the certifiable robustness problem by enforcing the following conditions: (i) $E$ accurately models the environment and (ii) $\safetyLoss(\pi, E)$ measures the ``provable robustness'' of $\pi$. Condition (i) is handled by alternating model updates with policy updates, in the style of Dyna~\cite{sutton1990dyna}, so we will focus on condition (ii).

The computation of $\safetyLoss$ uses the same underlying abstract rollouts (Algorithm~\ref{alg:abs-rollout}) as the verifier described in Algorithm~\ref{alg:proof}. Once again, this algorithm estimates the reward achieved by a policy under worst-case adversarial perturbations but average-case policy actions and environment transitions. We then define the robustness loss as the difference between the nominal loss $R^o$ and the \emph{provable} lower bound on the worst-case loss $R_\text{min}$. Now as long as $\safetyLoss < \Delta$, we satisfy the definition of robustness given in Section~\ref{sec:prob} for that specific trace. Repeating these gradient updates gives an approximation of the average-case behavior which is considered in Algorithm~\ref{alg:proof}.

%% file: sec_theory.tex
Now we explore some key theoretical properties of \sys. Proofs are deferred to Appendix~\ref{app:proofs}.

\begin{theorem}\label{thm:sound}
Assume the environment transition distribution is $P(s' \mid s, a) = \norm(\mu_P(s, a), \mSigma_P)$ and the environment model is $E(s' \mid s, a) = \norm(\mu_E(s, a), \mSigma_E)$ with $\mSigma_P, \mSigma_E$ diagonal. Further, we assume that the model satisfies a PAC-style guarantee: for any state $s$, action $a$, and $\epsilon \in \states$, $| (\mu_P(s, a) + \mSigma_P^{1/2} \epsilon) - (\mu_E(s, a) + \mSigma_E^{1/2} \epsilon) | \le \varepsilon_E$ with probability at least $1 - \delta_E$. For any policy $\pi$, let the result of Algorithm~\ref{alg:proof} be $\hat{R}^\#$ and let the reward of $\pi$ under the optimal adversary $\adv^*$ be $R$. Then for any $\delta > 0$ with probability at least $1 - \delta$, we have
\begin{align*}
R \ge & \hat{R}^\# - \frac{1}{\sqrt{\delta}} \sqrt{\frac{\var \left [ R^\# \right ]}{N}} - \left ( 1 - \left ( 1 - \delta_E \right )^T \right ) C.
\end{align*}
where $C$ is a constant (see Appendix~\ref{app:proofs} for details of $C$).
\end{theorem}

Theorem~\ref{thm:sound} shows that our checker is a valid (probabilistic) proof strategy for determining if a policy is robust. That is, if we use Algorithm~\ref{alg:proof} to measure the reward of a policy under perturbation, the result is a lower bound of the true worst-case reward (minus a constant) with high probability, assuming an accurate environment model. The bound in Theorem~\ref{thm:sound} gives some interesting insights. First, the bound grows as $\delta$ shrinks, so we pay the price of a looser bound as we consider higher confidence levels. Second, the bound depends on the variance of the abstract reward and the number of samples in an intuitive way --- higher variance makes it harder to measure the true reward, and more samples make the bound tighter. Third, as $\delta_E$ increases, the last term of the bound grows, indicating that a less accurate environment model leads to a looser bound. Finally, the bound grows with $T$, indicating that over longer time horizons, our reward measurement gets less accurate. This is consistent with the intuition that the environment model may drift away from the true environment over long rollouts.

\begin{theorem}\label{thm:convergence}
Algorithm~\ref{alg:main} converges to a policy $\pi$ which is robust and verifiable by Algorithm~\ref{alg:proof}.
\end{theorem}
\vspace{-5pt}

Intuitively, the theorem shows that Algorithm~\ref{alg:main} solves the certifiable robustness problem, i.e., it converges to a policy that passes the check by Algorithm~\ref{alg:proof}. The proof is straightforward because Algorithm~\ref{alg:main} is a standard primal-dual approach to solve the constrained optimization problem outlined in Equation~\ref{eq:goal} \cite{nandwani2019lagrange}.

%% file: sec_evaluation.tex
We study the following experimental questions:

\begin{enumerate}[label=\upshape\bfseries RQ\arabic*:, wide = 0pt, itemsep=0pt,topsep=0pt] 
    \item Can \sys learn policies with nontrivial 
certified reward bounds? 
    \item Do the certified bounds for \sys beat those for other (non-certified) robust RL methods? 
    \item How is \sys's performance on empirical adversarial inputs?
    \item How does \sys's model-based training approach affect performance?
\end{enumerate}

\textbf{Environments and Setup.}
Our experiments consider 
$l_\infty$-norms perturbation of the state with radius $\epsilon$: $B_p(s, \epsilon) := \{s' | \| s' - s \| \le \epsilon \}$.
We implement \sys on top of the 
MBPO \cite{janner2019trust} model-based RL algorithm using the implementation from \cite{Pineda2021MBRL}. For training, we use Interval Bound Propagation (IBP) \cite{gowal2018effectiveness} as a scalable abstract interpretation mechanism to compute the layer-wise bounds for the neural networks, \red{where all the abstract states are represented as intervals per dimension}. \red{More details of the abstract transition are omitted to \cref{app:ai-bound}.} During the evaluation, we use CROWN \cite{zhang2018efficient}, a more computationally expensive but tighter bound propagation method based on IBP. During training, we use a $\epsilon$-schedule \cite{gowal2018effectiveness, zhang2020robust} to slowly increase the $\epsilon_t$ at each epoch within the perturbation budget until reaching $\epsilon$. Note that the policies take action stochastically during training, but we set them to be deterministic during evaluation. 

We experiment on four MuJoCo environments in OpenAI Gym \cite{openai}. For \sys, we use the same hyperparameters for the base RL algorithms as in \cite{Pineda2021MBRL} without further tuning. Specifically, we do not use an ensemble of dynamics models. Instead, we use a single dynamic model, which is the case when the ensemble is of size 1. We use Gaussian distribution as the independent noise distribution, $f_\pi(a), f_E (s')$ for both policy and model in the experiments. Concretely, the output of our policies are the parameters $\mu_\pi$, $\mSigma_\pi$ of a Gaussian, with $\mSigma_\pi$ being diagonal and independent of input state $s$. For the model, the output are the parameters $\mu_E$, $\mSigma_E$ of a Gaussian, with $\mSigma_E$ being diagonal and independent of input $s, a$. \red{The model error $\varepsilon_E$ is measured and considered following \cref{alg:abs-rollout}. We defer the detailed descriptions of the model error and its construction details to \cref{app:model-error}.}

We compare \sys with the following methods: (1) MBPO \cite{janner2019trust}, our base RL algorithm. (2) SA-PPO \cite{zhang2020robust}, a robust RL algorithm bounding per-step action distance. (3) RADIAL-PPO \cite{oikarinen2021robust}, a robust RL algorithm using lower bound PPO loss to update the policy. (4) \sys-Separate Sampler(\sys-SS), an ablation of \sys. In \sys, we update the policy loss $\policyLoss$ with the data sampled from the rollout between the learned model and the policy. While in \sys-SS, the data for $\policyLoss$ is sampled from the rollout between the environment and the policy. The $\epsilon_{\text{train}}$ is 0.075, 0.05, 0.075, 0.05 for Hopper, Walker2d, HalfCheetah, and Ant for \sys, \sys-SS, SA-PPO, and RADIAL-PPO in this section for consistency with baselines. \red{Detailed training setup are available in \Cref{app:training}}.

\textbf{Evaluation.} We evaluate the performance of policies with two metrics: (i). \wcar, which was formally defined in \Cref{alg:proof} for certified performance. (ii). total reward under MAD attacks \cite{zhang2020robust} for empirical performance.

\begin{figure*}[t]
    \centering
    \includegraphics[scale=0.45, bb=0 0 850 240]{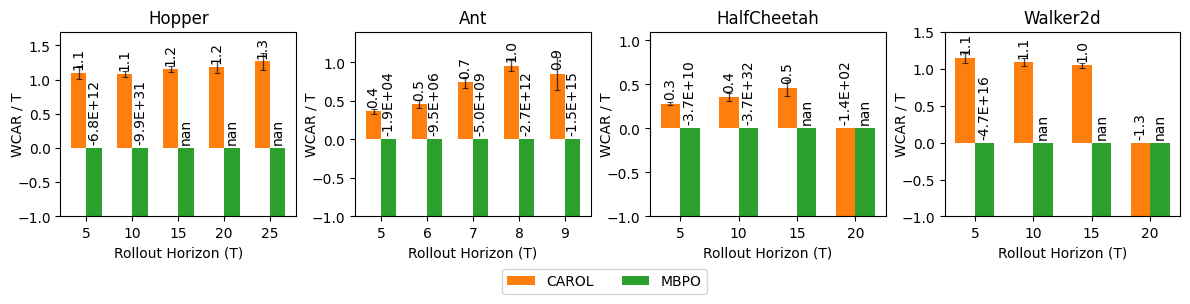}
    \caption{Certified performance of policies $\pi$ with the \textit{learned-together} model, $E$. To have a fair comparison across different horizons, we quantify the certified performance by $\text{\wcar}/T$, where \wcar is formally defined in \cref{alg:proof} and $T$ is the rollout horizon in \cref{alg:abs-rollout}. Each bar is an average of 25 starting states. \textit{nan} denotes \textit{not a number}, which means that ($\pi, E$) is not certifiable by a third-party verifier \cite{zhang2018efficient}. We use negative infinity to exhibit \textit{nan}'s value. A higher value indicates a better certified \textit{worst-case} performance. 
    The results are based on $\varepsilon_E$ with a $1 - \delta_{E}$ of $0.9$.}
    \label{fig:wcar-overall-0.90}
\vspace{-10pt}
\end{figure*}
\begin{figure*}[t]
    \centering
    \includegraphics[scale=0.45, bb=0 0 850 240]{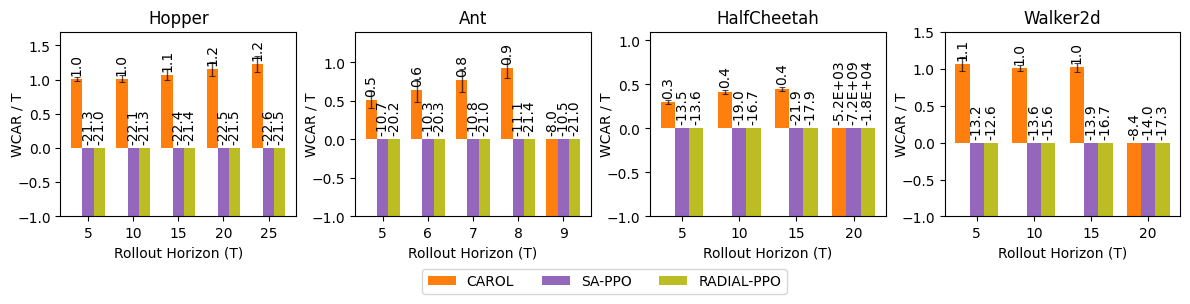}
    \caption{Certified performance of policies $\pi$ under a set of \textit{separately learned} models, $\{E_i\}$. Each bar averages the learned policies on each $E_i$ of 25 starting states. The results are based on $\varepsilon_E$ with a $1 - \delta_{E}$ of $0.9$.}
    \label{fig:wcar-separate-0.90}
\vspace{-10pt}
\end{figure*}

\textbf{RQ1: Certified Performance with Learned-together Certificate.} After training, we get a policy, $\pi$, and an environment model, $E$, trained with the policy. Then, we evaluate the \wcar following \Cref{alg:proof} with $\pi$ and $E$. Note that we use an $\epsilon_{\text{test}}=\frac{1}{255}$ for the evaluation of provability as certifying over long-horizon traces of neural network models tightly is a challenging task for abstract interpreters due to accumulated approximation error. The proof becomes more challenging as the horizon increases, primarily due to the consideration of the adversary's behavior in the most unfavorable scenarios at each step, and the step-wise impact from the worst-case adversary accumulates. We vary the certified horizon under the $\epsilon_{\text{test}}$ to exhibit the certified performance.  

\Cref{fig:wcar-overall-0.90} exhibits the certified performance of \sys. Both \sys and MBPO are evaluated with the model trained together. We are able to train a policy with better certified accumulative reward under the worst attacks compared to the base algorithm, MBPO, which does not use the regularization $\safetyLoss$. As the time horizon increases, it becomes harder to certify the accumulative reward. For example, in Ant and HalfCheetah, \sys is not able to give a good certified performance when the horizon reaches $10$ and $20$ respectively because of the accumulative influence from the worst-case attack and the overapproximation from the abstract interpreter. We also highlight that Ant is a challenging task for certification due to the high-dimensional state space.

\textbf{RQ2: Comparison of Certified Performance with Other Methods.} We compare \sys with two robust RL methods, SA-PPO \cite{zhang2020robust} and RADIAL-PPO \cite{oikarinen2021robust}, which both bound the per-step performance of the policy during training. SA-PPO bounds the per-step action deviation under perturbation, and RADIAL-PPO bounds the one-step loss under perturbation. \textit{To have a fair comparison of the certified performance of policies and alleviate the impact from model error bias across methods, we separately train 5 additional environment models, $\{E_i\}$, with the trajectory datasets unrolled from 5 additional random policies and the environment.} \red{Details are deferred to \cref{app:model-error}}. We truncate \sys by extracting the policies from training and certify them with these separately trained environment models. This setting is not completely inline with \sys's learned certificate and verification (see RQ1) but is designed for a fair comparison across policies.

As shown in \Cref{fig:wcar-separate-0.90}, the \sys's certified performance with separately trained models is slightly worse yet comparable to its performance when using learned-together certificates. Compared with non-certified RL policies, \sys consistently exhibits better certifiable performance. It is worth noting that \sys is able to provide worst-case rewards over time for benchmarks aligning with the reward mechanisms used in these environments. \red{We show the abstract trace lower bound ($R_{\text{min}}$) sampled from trajectories in \cref{app:quali-example}.} These results demonstrate that \sys is able to provide reasonable certified performance, while the other methods, which are not specifically designed for worst-case accumulative reward certification, struggle to attain the same goal.

\begin{wraptable}{r}{3.5in}
\vspace{-10pt}
{\small{
\begin{tabular}{@{}llll@{}} 
\toprule
& & Nominal & Attack (MAD) \\
Environment & Model & $\epsilon=0$ & $\epsilon=\epsilon_{\text{train}}$ \\
\midrule
\multirow{4}{*}{\makecell{Hopper \\ ($\epsilon_{\text{train}}=0.075$)}} & MBPO & 3246.0$\pm$76.1 & 2874.2$\pm$203.4\\
                                                          & SA-PPO & 3423.9$\pm$164.2 & \textbf{3213.8$\pm$284.8}\\
                                                          & RADIAL-PPO & \textbf{3547.0$\pm$166.9} & 3100.3$\pm$368.3\\ \cmidrule{2-4}
                                                          & CAROL & 3290.1$\pm$104.9 & 3201.4$\pm$100.5\\
\midrule
\multirow{4}{*}{\makecell{Ant \\ ($\epsilon_{\text{train}}=0.05$)}} & MBPO & 4051.9$\pm$526.2 & 406.2$\pm$83.5\\
                                                          & SA-PPO & 5368.8$\pm$96.4& 5327.4$\pm$112.7\\
                                                          & RADIAL-PPO & 4694.1$\pm$219.5 & 4478.9$\pm$232.8\\ \cmidrule{2-4}
                                                          & CAROL & \textbf{5696.6$\pm$277.9} & \textbf{5362.2$\pm$242.8}\\
\midrule
\multirow{4}{*}{\makecell{HalfCheetah \\ ($\epsilon_{\text{train}}=0.075$)}} & MBPO & \textbf{7706.3$\pm$710.1}& 2314.6$\pm$566.7\\
                                                          & SA-PPO & 3193.9$\pm$650.7& 3231.6$\pm$659.9\\
                                                          & RADIAL-PPO & 3686.5$\pm$439.2 & 3409.6$\pm$683.9\\ \cmidrule{2-4}
                                                          & CAROL & 5821.5$\pm$2401.9& \textbf{3961.6$\pm$899.5}\\
\midrule
\multirow{4}{*}{\makecell{Walker2d \\ ($\epsilon_{\text{train}}=0.05$)}} & MBPO & 3815.6$\pm$211.9 & 3616.5$\pm$228.2\\
                                                          & SA-PPO &\textbf{4271.7$\pm$222.2} & \textbf{4444.4$\pm$286.0}\\
                                                          & RADIAL-PPO & 2935.1$\pm$272.1 & 3022.6$\pm$381.7 \\ \cmidrule{2-4}
                                                          & CAROL & 3784.4$\pm$329.1 & 3774.3$\pm$260.3\\
\bottomrule
\end{tabular}
}
}
\caption{Average episodic reward $\pm$ standard deviation over 100 episodes on three baselines and \sys. We show natural rewards (under no attack) and rewards under adversarial attacks. The best results over all methods are in bold.}
\label{tab:adv-results}
\vspace{-10pt}
\end{wraptable}

\textbf{RQ3: Comparison of Empirical Performance with Other Methods.}
Usually, there is a trade-off between certified robustness and empirical robustness. One can get good provability but may sacrifice empirical rewards. We show that policy from our algorithm shows comparable natural rewards (without attack) and adversarial rewards compared with other methods. In \Cref{tab:adv-results}, we show results on 4 environments and comparison with MBPO, SA-PPO, and RADIAL-PPO. The policies are the same ones evaluated for RQ1 and RQ2. For each environment, we compare the performance under MAD attacks \cite{zhang2020robust}. \sys outperforms other methods on Ant and HalfCheetah under attacks when the base algorithm, MBPO, is extremely not robust. For Hopper, \sys has comparable adversarial rewards with the best methods. \sys's reward is worse on Walker2d though still reasonable.

\begin{figure*}[t]
\vspace{-15pt}
    \centering
    \includegraphics[scale=0.4, bb=0 0 1000 240]{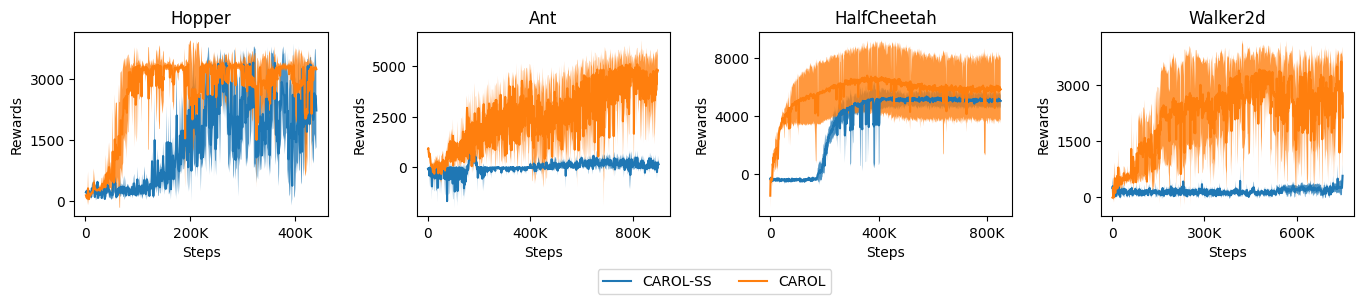}
    \caption{Training Curves of \sys and \sys-SS. The solid lines in the graph show the average natural rewards of five training trials, and the shaded areas represent the standard deviation among those trials.}
    \label{fig:curve}
\vspace{-15pt}
\end{figure*}

\textbf{RQ4: Impact of Model-Based Training.}
In this part, we investigate the impact of our design choices for $\policyLoss$ on performance. We compare our framework, \sys, with an ablation of it, \sys-SS, to understand how rollout with the learned model for $\policyLoss$ matters in \sys. We present a comparison of performance during training, as shown in \Cref{fig:curve}. In the implementation, we set a smoother $\epsilon$-schedule for \sys-SS by allowing \sys-SS to take longer steps from $\epsilon=0$ to the target $\epsilon$. These results show that \sys converges much faster while achieving a comparable or better final performance due to the benefits of the sample efficiency of MBRL. Additionally, the consistency between the rollout datasets for $\policyLoss$ and the ones for $\safetyLoss$ also leads to a better natural reward at convergence in training.

%% file: sec_related.tex
\textbf{Adversarial RL.} \red{Adversarial attacks on RL systems have been extensively studied. Specific attacks include adversarial perturbations on agents' observations or actions \cite{huang2017adversarial, lin2017tactics, weng2019toward}, adversarial disturbance forces to the system \cite{pinto2017robust}, and other adversarial policies in a multiagent setting \cite{gleave2019adversarial}. Most recently, \cite{zhang2021robust} and \cite{sun2021strongest} consider an optimal adversary and propose methods to train agents together with a learned adversary in an online way to achieve a better adversarial reward.}

\textbf{Robust RL and Certifiable Robustness in RL.} Multiple robust training methods have been applied to deep RL. \red{Mankowitz et al. \cite{mankowitz2019robust} explore a broader adversarial setting related to model disturbances and model uncertainty.} Fischer et al. \cite{fischer2019online} leverage additional student networks to help the robust Q learning, and Everett et al. \cite{everett2021certifiable} enhance an agent's robustness during testing time by computing the lower bound of each action's Q value at each step. Zhang et al. \cite{zhang2020robust} and Oikarinen et al. \cite{oikarinen2021robust} leverage a bound propagation technique in a loss regularizer to encourage the agent to either follow its original actions or optimize over a loss lower bound. While these efforts achieve robustness by deterministic certification techniques for neural networks \cite{gowal2018effectiveness, xu2020automatic}, they mainly focus on the step-wise certification and are not able to give robustness certification if the impact from attacks accumulates across multiple steps. \sys differs from these papers by offering certified robustness for the aggregate reward in an episode. We know of only two recent efforts that study robustness certification for cumulative rewards.
The first, by Wu et al. \cite{wu2021crop}, gives a framework for certification rather than certified learning. The second, by Kumar et al. \cite{kumar2021policy},  proposes a certified learning algorithm under the assumption that the adversarial perturbation is smoothed using random noise. The attack model here is weaker than the adversarial model assumed by \sys and most other work on adversarial learning.  

\textbf{Certified RL.}
Safe control with learned certificates is an active field \cite{dawson2022safesurvey}. A few efforts in this space have considered controllers discovered through RL. Many works use a given certificate with strong control-theoretic priors to constrain the actions of an RL agent \cite{cheng2019end, li2019temporal, cheng2019control} or assume the full knowledge of the environment to yield the certificate during the training of an agent \cite{yang2021safe}. Chow et al. \cite{chow2019lyapunov,chow2018lyapunov} attempt to derive certificates from the structure of the constrained Markov decision process \cite{altman1999constrained} for the safe control problems. Chang et al. \cite{chang2021stabilizing} incorporate Lyapunov methods in deep RL to learn a neural Lyapunov critic function to improve the stability of an RL agent. We differ from this work by focusing on adversarial robustness rather than stability.

%% file: sec_conclusion.tex
We have presented \sys, the first RL framework with certifiable episode-level robustness guarantees. Our approach is based on a new combination of model-based RL and abstract interpretation. We have given a theoretical analysis to justify the approach and validated it empirically in four challenging continuous control tasks.

\red{A key challenge in \sys is that our abstract interpreter may not be sufficiently precise, and attempts to increase precision may compromise scalability. Future research should work to address this issue with more accurate and scalable verification techniques.
Also, as seen in \cref{fig:curve}, there is variance in the results due to the use of a single dynamics model. To address this issue, future work should explore ways to incorporate certificates across ensembles of models. Moreover, because abstract interpretation of probabilistic systems is difficult, our approach assumes that the randomness in the environment transitions is state-independent. Future work should try to eliminate this assumption through abstract interpreters tailored to probabilistic systems. Finally, we focus on the state-adversarial setting, which of practical applications. However, there are broader adversarial settings encompassing model disturbances and model uncertainty. Exploring these areas within the context of certified learning is also an interesting future direction. A broader discussion of limitations appears in \cref{app:discuss}.}

%% file: Appendix/symbols.tex
\section{Symbols}
We give a summary of the symbols used in this paper below.
\begin{center}
\scalebox{0.75}{
\begin{tabular}{ c|c } 
\toprule
Definition & Symbol/Notation \\
\midrule
Policy & $\pi_{\psi}, \pi$  \\ 
Model of the environment & $E_{\theta}, E$  \\ 
Environment transition & $P$ \\
Parameters (mean, covariance) for Gaussian distribution for environment, model, and policy &  $\mu_{P}, \mSigma_{P}, \mu_{E}, \mSigma_{E}, \mu_{\pi}, \mSigma_{\pi}$ \\
Distribution representing noise for environment, model, and policy & $f_P, f_E, f_{\pi}$ \\
Adversary & $\adv$ \\
Regular policy loss & $\policyLoss$ \\
Robustness loss & $\safetyLoss$ \\
\midrule 
Lipschitz constants for the environment model and policy mean & $L_E$, $L_{\pi}$ \\
Model error & $\varepsilon_E$ \\
PAC bound probability & $\delta_E$ \\
\midrule
Abstract lifts & $\cdot^\#$ \\
Abstraction & $\alpha$ \\
Concretization & $\beta$ \\
\midrule
Horizon for training and testing & $T$, $T_{\text{train}}$, $T_{\text{test}}$\\
Disturbance for training and testing & $\epsilon$, $\epsilon_{\text{train}}$, $\epsilon_{\text{test}}$ \\
\bottomrule
\end{tabular}
}
\end{center}

%% file: Appendix/proofs.tex
\section{Proofs}\label{app:proofs}

In this section we present proofs of the theorems from Section~\ref{sec:the}.

\begin{assumption}\label{asm:bounded-horizon}
The horizon of the MDP is bounded by $T$.
\end{assumption}

\begin{assumption}\label{asm:env-gaussian}
The environment transition distribution has the form $P \left ( s' \mid s, a \right ) = \norm \left ( \mu_P \left (s, a \right ), \mSigma_P \right )$ with $\mSigma_P$ diagonal and the environment model $E$ has the form $E \left (s' \mid s, a \right ) = \norm \left ( \mu_E(s, a), \mSigma_E \right )$ with $\mSigma_E$ diagonal.
\end{assumption}

\begin{assumption}\label{asm:env-close}
There exist values $\varepsilon_E$ and $\delta_E$ such that for all $s, a$ and for any fixed $e$ with probability at least $1 - \delta_E$, $\left \| \left ( \mu_P \left ( s, a \right ) + \mSigma_P^{1/2} e \right ) - \left ( \mu_E \left ( s, a \right ) + \mSigma_E^{1/2} e \right ) \right \| \le \varepsilon_E$. Further, there exists some $d_E$ such that for all $s, a$, $\left \|  \left ( \mu_P \left ( s, a \right ) + \mSigma_P^{1/2} e \right ) - \left ( \mu_E \left ( s, a \right ) + \mSigma_E^{1/2} e \right ) \right \| \le d_E$.
\end{assumption}

\begin{assumption}\label{asm:lipschitz}
The environment model mean function $\mu_E \left ( s, a \right )$ is $L_E$-Lipschitz continuous, the immediate reward function $r \left ( s, a \right )$ is $L_r$-Lipschitz continuous, and the policy mean $\mu_\pi \left ( s \right )$ is $L_\pi$-Lipschitz continuous.
\end{assumption}

\begin{assumption}\label{asm:adv-center}
For all $s \in \states$, we have $s \in B \left ( s \right )$. That is, the adversary may choose not to perturb any state.
\end{assumption}

\begin{reptheorem}[\ref{thm:sound}]
For any policy $\pi$, let the result of Algorithm~\ref{alg:proof} be $\hat{R}^\#$, let $\adv^*$ be the optimal adversary (i.e., for all $\adv \in \advset_B$, $R \left ( \pi \circ \adv^* \right ) \le R \left ( \pi \circ \adv \right )$), and let the reward of $\pi \circ \adv^*$ be $R$. Then for any $\delta > 0$ with probability at least $1 - \delta$, we have
\[ R \ge \hat{R}^\# \left ( \tau \right ) - \frac{1}{\sqrt{\delta}} \sqrt{\frac{\var \left [ R^\# \right ]}{N}} - \left ( 1 - \left ( 1 - \delta_E \right )^T \right ) L_r (1 + L_\pi) d_E \frac{ \left ( L_E L_\pi \right )^T + \left ( 1 - L_E L_\pi \right ) T - 1}{ \left ( 1 - L_E L_\pi \right )^2}. \]
\end{reptheorem}
\begin{proof}
Recall that the environment transition $P$ and policy $\pi$ are assumed to be separable, i.e., $P \left (s' \mid s, a \right ) = \mu_P \left (s, a \right ) + f_P \left ( s' \right )$ and $\pi \left ( a \mid s \right ) = \mu_\pi \left ( s \right ) + f_\pi \left ( a \right )$ with $\mu_P$ and $\mu_\pi$ deterministic. As a result, a trajectory under policy $\pi \circ \adv^*$ may be written $\tau = s_0, a_0, s_1, a_1, \ldots, s_n, a_n$ where $s_0 \sim \states_0$, each $a_i = \mu_\pi \left ( \adv^* \left ( s_i \right ) \right ) + e^\pi_i$ for $e^\pi_i \sim f_\pi \left ( a \right )$, and each $s_i = \mu_P \left ( s_{i-1}, a_{i-1} \right ) + e^P_i$ for $e^P_i \sim f_P \left ( s' \right )$. By Assumption~\ref{asm:env-gaussian}, we know that $e^P_i \sim \norm \left ( \mZero, \mSigma_P \right )$ so that $e^P_i = \mSigma_P^{1/2} e_i$ where $e_i \sim \norm \left ( 0, \mI \right )$. In particular, because each trajectory $\tau$ is uniquely determined by $s_0, \{e^\pi_i\}_{i=0}^n, \{e_i\}_{i=1}^n$, we can write the reward of $\pi \circ \adv^*$ as
\[R \left ( \pi \circ \adv^* \right ) = \expec_{s_0 \sim \states_0, \left \{ e^\pi_i \sim f_\pi \left ( a \right ) \right \}_{i=0}^n, \left \{ e_i \sim \norm \left ( \mZero, \mI \right ) \right \}_{i=1}^n} R \left ( \tau \right ) \]
Because this expectation ranges over the values of $s_0, \left \{ e^\pi_i \right \}_{i=0}^n, \left \{ e_i \right \}_{i=1}^n$, we will proceed by considering pairs of abstract and concrete trajectories unrolled with the same starting state and noise terms.

To do this, we analyze Algorithm~\ref{alg:abs-rollout} for some fixed $s_0, \left \{ e^\pi_i \right \}_{i=0}^n, \left \{ e_i \right \}_{i=1}^n$. Let $e^E_i = \mSigma_E^{1/2} e_i$. That is, given the same underlying sample from $\norm \left ( 0, \mI \right )$, $e^P_i$ is the noise in the true environment while $e^E_i$ is the noise in the modeled environment. We show by induction that for all $i$, $s_i \in \conc \left ( s_{\text{original}_i}^\# \right )$ with probability at least $ \left ( 1 - \delta_E \right )^i$. Note that, because abstract interpretation is sound, $s_0 \in \conc \left ( s_{\text{original}_0}^\# \right )$. Additionally, for all $i$ if $s_i \in \conc \left ( s_{\text{original}_i}^\# \right )$ then $\adv^* \left ( s_i \right ) \in \conc \left ( s_{\text{obs}_i}^\# \right )$. Moreover, since $e^\pi_i$ is fixed, we have
\[\pi \left ( s_{\text{obs}_i}^\# \right ) = \mu_\pi \left ( s_{\text{obs}_i}^\# \right ) + e^\pi_i \]
so that $\pi \left ( \adv^* \left ( s \right ) \right ) \in \conc \left ( a_i^\# \right )$. Similarly, because $e^E_i$ is fixed, let $\Delta_E = \alpha \left ( \left \{ x \mid \left \| x \right \| \le \varepsilon_E \right \} \right )$ and we have
\[ E \left ( s_{\text{original}_i}^\#, a_i^\# \right ) + \Delta_E = \mu_E \left ( s_{\text{original}_i}^\#, a_i^\# \right ) + e^E_i + \Delta_E . \]
By the induction hypothesis, we know that $s_{i-1} \in \conc \left ( s_{\text{original}_{i-1}}^\# \right )$ with probability at least $\left ( 1 - \delta_E \right )^{i-1}$ and therefore $a_{i-1} \in \conc \left ( a_{i-1}^\# \right )$. By Assumption~\ref{asm:env-close}, we have that $\left \| \left ( \mu_P \left ( s, a \right ) + e^P_i \right ) -  \left ( \mu_E \left ( s, a \right ) + e^E_i \right ) \right \| < \varepsilon_E$ with probability at least $1 - \delta_E$. In particular, $\left ( \mu_P \left ( s, a \right ) + \varepsilon^P_i \right ) - \left ( \mu_E \left ( s, a \right ) + \varepsilon^E_i \right ) \in \Delta_E$, so that
$\mu_P \left ( s, a \right ) + e^P_i \in \conc \left ( E \left ( s_{\text{original}_i}^\#, a_i^\# \right ) + \Delta_E \right )$. Then with probability at least $1 - \delta_E$, if $s_{i-1} \in \conc \left ( s_{\text{original}_{i-1}}^\# \right )$ then $s_i \in \conc \left ( s_{\text{original}_i}^\# \right )$. As a result, $s_i \in \conc \left ( s_{\text{original}_i}^\# \right )$ with probability at least $\left ( 1 - \delta_E \right )^i$. In particular, by Assumption~\ref{asm:bounded-horizon}, $n \le T$ so that for a fixed $\tau$ defined by $s_0, \left \{ e^\pi_i \right \}_{i=0}^n, \left \{ e_i \right \}_{i=1}^n$, we have that with probability at least $\left ( 1 - \delta_E \right )^T$, Algorithm~\ref{alg:abs-rollout} returns a lower bound on $R \left ( \tau \right)$.

Now we consider the case where Algorithm~\ref{alg:abs-rollout} does \emph{not} return a lower bound of $R \left ( \tau \right )$. In this case, we show (again by induction) that for all $0 \le i \le T$, there exists a point $s_i' \in \conc \left ( s_{\text{original}_i}^\# \right )$ such that
\[ \left \| s_i - s_i' \right \| \le \sum_{j=0}^{i-1} \left ( L_E L_\pi \right )^j d_E = d_E \left ( \frac{1 - \left ( L_E L_\pi \right )^{i-1}}{1 - L_E L_\pi} \right ) \]
(when $\sum_{j=0}^{-1} \left ( L_E L_\pi \right )^j d_E$ is taken to be zero). First, note that $s_0 \in \conc \left ( s_{\text{original}_0}^\# \right )$, so the base case is trivially true. Now by the induction hypothesis we have that there exists some $s_{i-1}' \in \conc \left ( s_{\text{original}_{i-1}}^\# \right )$ with $\left \| s_{i-1} - s_{i-1}' \right \| \le \sum_{j=0}^{i-2} \left ( L_E L_\pi \right )^j d_E$. Notice that by Assumption~\ref{asm:adv-center}, we also have $s_{i-1}' \in \conc \left ( s_{\text{obs}_{i-1}}^\# \right )$. Now because abstract interpretation is sound, we have that $\mu_\pi \left ( s_{i-1}' \right ) + e^\pi_{i-1} \in \conc \left ( a_{i-1}^\# \right )$ and by Assumption~\ref{asm:lipschitz}, $\| \mu_\pi \left ( s_{i-1} \right ) - \mu_\pi \left ( s_{i-1}' \right ) \| \le L_\pi \sum_{j=0}^{i-2} \left ( L_E L_\pi \right )^j d_E$. Similarly, we have $\mu_E \left ( s_{i-1}', \mu_\pi \left ( s_{i-1}' \right ) + e^\pi_{i-1} \right ) + e^E_i \in \conc \left ( s_{\text{original}_i}^\# \right )$, and $\left \| \mu_E \left ( s_{i-1}, \mu_\pi \left ( s_{i-1} \right ) + e^\pi_{i-1} \right ) - \mu_E \left ( s_{i-1}', \mu_\pi \left ( s_{i-1}' \right ) + e^\pi_{i-1} \right ) \right \| \le L_E L_\pi \sum_{j=0}^{i-2} \left ( L_E L_\pi \right )^j d_E$. Let $\hat{s}_i = \mu_E \left ( s_{i-1}, \mu_\pi\left ( s_{i-1} \right ) + \varepsilon_{i-1}^\pi \right ) + \varepsilon^E_i$. Then by Assumption~\ref{asm:env-close}, we have $\left \| \hat{s}_i - s_i \right \| \le d_E$, so that in particular $\left \| s_i - \mu_E \left ( s_{i-1}', \mu_\pi \left ( s_{i-1}' \right ) + e^\pi_{i-1} \right ) + \varepsilon^E_i \right \| \le d_E + L_E L_\pi \sum_{j=0}^{i-2} \left ( L_E L_\pi \right )^j d_E$. Letting $s_i' = \mu_E \left ( s_{i-1}', \mu_\pi \left ( s_{i-1}' \right ) + e^\pi_{i-1} \right ) + \varepsilon^P_i$, we have the desired result.

We use this result to bound the difference in reward between the abstract and concrete rollouts when Algorithm~\ref{alg:abs-rollout} does not return a lower bound. For each $i$, because $s_i' \in \conc \left ( s_{\text{original}_i}^\# \right )$ and $\mu_\pi \left ( s_i' \right ) + e^\pi_i \in a_i^\#$, we define $r_i' = r \left ( s_i', \mu_\pi \left ( s_i' \right ) + e^\pi_i \right )$ and we know that $r' \in r_i^\#$. Because $\left \| s_i - s_i' \right \| \le d_E \left ( \frac{1 - \left ( L_E L_\pi \right )^{i-1}}{1 - L_E L_\pi} \right )$ we have $\left \| a_i - a_i' \right \| \le L_\pi d_E \left ( \frac{1 - \left ( L_E L_\pi \right )^{i-1}}{1 - L_E L_\pi} \right ) $ and $\left | r \left ( s_i, a_i \right ) - r_i' \right | \le L_r (1 + L_\pi) d_E \left ( \frac{1 - \left ( L_E L_\pi \right )^{i-1}}{1 - L_E L_\pi} \right )$. In particular,
let $R' = \sum_i r_i'$ and then
\begin{align*}
\left | R \left ( \tau \right ) - R' \right | &\le \sum_{i=1}^T L_r (1 + L_\pi) d_E \left ( \frac{1 - \left ( L_E L_\pi \right )^{i-1}}{1 - L_E L_\pi} \right ) \\
&= L_r (1 + L_\pi) d_E \frac{ \left ( L_E L_\pi \right )^T + \left ( 1 - L_E L_\pi \right ) T - 1}{\left ( 1 - L_E L_\pi \right )^2}.
\end{align*}

We now combine these two cases to bound the expected difference between the reward returned by Algorithm~\ref{alg:abs-rollout}, denoted $R^\# \left ( \tau \right )$, and the reward of $\tau$. Let $D = R^\# \left ( \tau \right ) - R \left ( \tau \right )$ be a random variable representing this difference. Then with probability at least $\left ( 1 - \delta_E \right )^T$, $D \le 0$ and in all other cases (i.e., with probability no greater than $1 - \left ( 1 - \delta_E \right )^T$), $D \le L_R (1 + L_\pi) d_E \frac{\left ( L_E L_\pi \right )^T + \left ( 1 - L_E L_\pi \right ) T - 1}{\left ( 1 - L_E L_\pi \right )^2}$. In particular then,
\[ \expec \left [ D \right ] \le \left ( 1 - \left ( 1 - \delta_E \right )^T \right ) L_r (1 + L_\pi) d_E \frac{\left ( L_E L_\pi \right )^T + \left ( 1 - L_E L_\pi \right ) T - 1}{\left ( 1 - L_E L_\pi \right )^2}. \]
By definition $\expec \left [ R^\# \left ( \tau \right ) \right ] = \expec \left [ R \left ( \tau \right ) \right ] + \expec \left [ D \right ]$. Therefore, we have
\begin{equation}\label{eq:difference-bound}
\begin{split}
\expec \left [ R \left ( \tau \right ) \right ] &= \expec \left [ R^\# \left ( \tau \right ) \right ] - \expec \left [ D \right ] \ge \expec \left [ R^\# \left ( \tau \right ) \right ] \\
&- \left ( 1 - \left ( 1 - \delta_E \right )^T \right ) L_r (1 + L_\pi) d_E \frac{\left ( L_E L_\pi \right )^T + \left ( 1 - L_E L_\pi \right ) T - 1}{\left ( 1 - L_E L_\pi \right )^2}.
\end{split}
\end{equation}

Algorithm~\ref{alg:main} approximates $\expec \left [ R^\# \left ( \tau \right ) \right ]$ by sampling $N$ values. Let $\hat{R}^\# \left ( \tau \right )$ be the measured mean and recall $\expec \left [ \hat{R}^\# \left ( \tau \right ) \right ] = \expec \left [ R^\# \left ( \tau \right ) \right ]$ and $\var \left [ \hat{R}^\# \left ( \tau \right ) \right ] = \var \left [ R^\# \left ( \tau \right ) \right ] / N$. Then by Chebyshev's inequality we have the for all $k > 0$, $\prob \left [ \left | \hat{R}^\# \left ( \tau \right ) - \expec \left [ R^\# \left (\tau \right ) \right ] \right | \ge k \sqrt{\var \left [ \hat{R}^\# \left ( \tau \right ) \right ]} \right ] \le 1 / k^2$. Then in particular, with probability at least $1 - 1 / k^2$,
\[ \hat{R}^\# \left ( \tau \right ) - k \sqrt{\frac{\var \left [ R^\# \left ( \tau \right ) \right ]}{N}} \le R^\# \left ( \tau \right ). \]
Combining this with Equation~\ref{eq:difference-bound} above and letting $k = 1 / \sqrt{\delta}$, we have with probability at least $1 - \delta$,
\begin{align*}
\expec \left [ R \left ( \tau \right ) \right ] &\ge \hat{R}^\# \left ( \tau \right ) - \frac{1}{\sqrt{\delta}} \sqrt{\frac{\var \left [ R^\# \left ( \tau \right ) \right ]}{N}} \\
&- \left ( 1 - \left ( 1 - \delta_E \right )^T \right ) L_r (1 + L_\pi) d_E \frac{ \left ( L_E L_\pi \right )^T + \left ( 1 - L_E L_\pi \right ) T - 1}{ \left ( 1 - L_E L_\pi \right )^2}.
\end{align*}
\end{proof}

While this paper focuses on continuous state and action spaces, we can extend our main theoretical result to discrete state and action spaces if the environment is deterministic. For this analysis, we maintain Assumptions~\ref{asm:bounded-horizon} and~\ref{asm:adv-center}, but we add a few new assumptions for the discrete setting.

\begin{assumption}\label{asm:discr-env-close}
    The environment model $E$ is deterministic and $E(s, a) = P(s, a)$ with probability at least $1 - \delta_E$.
\end{assumption}

\begin{assumption}\label{asm:discr-reward-bound}
    The single-step reward for any state $s$ and action $a$ is bounded by $r_{\min} \le r(s, a) \le r_{\max}$.
\end{assumption}

\begin{theorem}
For a deterministic policy $\pi$, let the result of Algorithm~\ref{alg:proof} be $\hat{R}^\#$, let $\nu^*$ be the optimal adversary, and let the reward of $\pi \circ \nu^*$ be $R$. Then for any $\delta$, with probability at least $1 - \delta$,
\[ R \ge \hat{R}^\# - \frac{1}{\sqrt{\delta}} \sqrt{ \frac{ \var[R^\#]}{N} } - T (r_{\max} - r_{\min}). \]
\end{theorem}
\begin{proof}
Consider a trajectory $\tau = s_0, a_0, s_1, a_1, \ldots, s_n, a_n$ where $s_0 \sim \states_0$, each $a_i = \pi(s_i)$, and each $s_{i+1} = P(s_i, a_i)$. Note that because the dynamics of the environment and the policy are deterministic, the only randomness in the trajectory comes from sampling the initial state. Then $R(\pi \circ \nu^*) = \E_{s_0 \sim \states_0} R(\tau)$. Similar to the proof of Theorem~\ref{thm:sound}, we proceed by considering pairs of abstract and concrete trajectories unrolled from the same starting state.

We show by induction that for all $i$, $s_i \in \conc \left ( s_{\text{original}_i}^\# \right )$ with probability at least ${ \left ( 1 - \delta_E \right ) }^i$. For the base case, note that because abstract interpretation is sound, $s_0 \in \conc \left ( s_{\text{original}_0}^\# \right )$. Additionally, for all $i$, if $s_i \in \conc \left ( s_{\text{original}_i}^\# \right )$ then $\nu^*(s_i) \in \conc \left ( s_{\text{obs}_i}^\# \right )$ and $\pi(\nu^*(s_i)) \in \conc \left ( a_i^\# \right )$. From the induction hypothesis, we have $s_{i-1} \in \conc \left ( s_{\text{original}_{i-1}}^\# \right )$ with probability at least $\left ( 1 - \delta_E \right )^{i-1}$. By Assumption~\ref{asm:discr-env-close}, we have $s_{i+1} = E(s_i, a_i)$ with probability at least $1 - \delta_E$. Thus if $s_{i-1} \in \conc \left ( s_{\text{original}_{i-1}}^\# \right )$ then with probability at least $1 - \delta_E$ we know $s_i \in \conc \left ( s_{\text{original}_i}^\# \right )$. Combined with the induction hypothesis, this implies that $s_i \in \conc \left ( s_{\text{original}_i}^\# \right )$ with probability at least ${(1 - \delta_E)}^i$. Now by Assumption~\ref{asm:bounded-horizon}, $n \le T$ so that for a fixed $\tau$ from a starting state $s_0$, we have that with probability at least ${(1 - \delta_E)}^T$, Algorithm~\ref{alg:proof} returns a lower bound on $R(\tau)$.

As in the proof of Theorem~\ref{thm:sound}, we now turn to the case where Algorithm~\ref{alg:proof} does not return a lower bound of $R(\tau)$. In this case, let $r_i$ be the true adversarial reward at time step $i$. Then by Assumption~\ref{asm:discr-reward-bound}, we have $r_i \ge r_{\min}$, and $\inf \conc \left ( r_i^\# \right ) \le r_{\max}$. Thus in particular, letting $R^\#(\tau)$ represent the bound returned by Algorithm~\ref{alg:proof}, we have $R^\#(\tau) - R(\tau) \le T (r_{\max} - r_{\min})$.

Now letting $D = R^\#(\tau) - R(\tau)$, we have that with probability at least ${(1 - \delta_E)}^T$, $D \le 0$ and in all other cases $D \le T (r_{\max} - r_{\min})$. In particular,
\[\E[D] \le \left (1 - { \left ( 1 - \delta_E \right )}^T \right ) T \left ( r_{\max} - r_{\min} \right ).\]
By definition, $\E[R^\#(\tau)] = \E[R(\tau)] + \E[D]$ so that
\[\E[R^\#(\tau)] = \E[R^\#(\tau)] - \left ( 1 - { \left ( 1 - \delta_E \right ) }^T \right ) T (r_{\max} - r_{\min}).\]
Following the same sampling argument we make in the proof of Theorem~\ref{thm:sound}, we have that for any $\delta$, with probability at least $1 - \delta$,
\[ \E[R(\tau)] \ge \hat{R}^\#(\tau) - \frac{1}{\sqrt{\delta}} \sqrt{ \frac{ \var[R^\#(\tau)]}{N} } - T (r_{\max} - r_{\min}). \]
\end{proof}

%% file: Appendix/implementation.tex
\section{Experiment Details and More Results} \label{app:imple}
\subsection{Training Details}\label{app:training}
We run our experiments on Quadro RTX 8000 and Nvidia T4 GPUs. We define the perturbation set $B(s)$ to be an $l_\infty$ norm perturbation of the state with radius $\epsilon$: $B_p(s, \epsilon) := \{s' | \| s' - s \| \le \epsilon \}$ in the experiments. We use a smoothed linear $\epsilon$-schedule during training as in \cite{zhang2020robust, oikarinen2021robust}. For the environments, we use the MuJoCo environments in OpenAI Gym \cite{openai}. We use Hopper-v2, HalfCheetah-v2, Ant-v2, Walker2d-v2 with 1000 trial lengths. 

\paragraph{Network Structures and Hyperparameters for Training} \label{app:hyperparameters}
For both policy networks and the model networks, we use the same network as in \cite{Pineda2021MBRL}. For both MBPO and \sys, we use the optimal hyperparameters in \cite{Pineda2021MBRL}. We set $T_{\text{train}}=1$ for all the training of \sys. We mainly set two additional parameters, regularization parameters and the $\epsilon$-schedule \cite{zhang2020robust, oikarinen2021robust, gowal2018effectiveness} parameters for \sys. The additional regularization parameter $\lambda$ to start with for regularizing $\safetyLoss$ is chosen in $\{0.1, 0.3, 0.5, 0.7, 1.0\}$. The $\epsilon$-schedule starts as an exponential growth from $\epsilon=10^{-12}$ and transitions smoothly into a linear schedule until reaching $\epsilon_{\text{train}}$. Then the schedule keeps $\epsilon_t=\epsilon_{\text{train}}$ for the rest of iterations. We set the temperature parameter controlling the exponential growth with $4.0$ for all experiments. We have two other parameters to control the $\epsilon$-schedule: \textit{endStep}, and \textit{finalStep}, where \textit{endStep} is the step where $\epsilon_t$ reaches $\epsilon_{\text{train}}$ and \textit{finalStep} is the steps for the total training. The $\textit{midStep} = 0.25 * \textit{endStep}$ is the turning point from exponential growth to linear growth. \Cref{tab:eschedule} shows the details of each parameter.

\begin{table}[h]
\centering
\begin{tabular}{ c c c c c} 
\toprule
Environments & Methods & \textit{endStep} & \textit{finalStep} \\ \midrule
\multirow{2}{*}{\makecell{Hopper}} & \sys &$4 \times 10^5$ & $5 \times 10^5$\\ 
                                    & \sys-SS & $4 \times 10^5$ & $5 \times 10^5$ \\ \midrule
\multirow{2}{*}{\makecell{Ant}} & \sys & $8 \times 10^5$ & $9 \times 10^5$\\ 
                                    & \sys-SS & $4 \times 10^6$ & $5 \times 10^6$ \\ \midrule
\multirow{2}{*}{\makecell{Walker2d}} & \sys & $7 \times 10^5$ & $7.5 \times 10^5$\\ 
                                    & \sys-SS & $1.5 \times 10^6$ & $2 \times 10^6$ \\ \midrule
\multirow{2}{*}{\makecell{HalfCheetah}} & \sys & $7.5 \times 10^5$ & $8.5 \times 10^5$\\ 
                                    & \sys-SS & $7.5 \times 10^5$ & $8.5 \times 10^5$ \\ \midrule
\end{tabular}
\caption{Parameters for $\epsilon$-schedule.}
\label{tab:eschedule}
\end{table}

\subsection{Certificate Usage and Model Error}\label{app:model-error}
In evaluation, we compare the certified performance of policies trained from different algorithms with a set of environment models. We measure the certified performance of policies in \cref{fig:wcar-overall-0.90} with the environment models trained together. \cref{fig:model-error} shows the model error distribution across methods. The datasets for \textit{CAROL-together} and \textit{MBPO-together} are collected during training and the datasets for  \textit{separate} contain the data for 800k steps. We use $80\%$ for training and $20\%$ for testing. In training, we use basic supervised learning to get environment models. The model architectures are the same as in \cite{Pineda2021MBRL}. We use $\varepsilon_E$ with the $1 - \delta_E$ of $0.90$ in \cref{sec:eval}.

\begin{figure*}[h]
    \centering
    \includegraphics[scale=0.35, bb=0 0 1000 540]{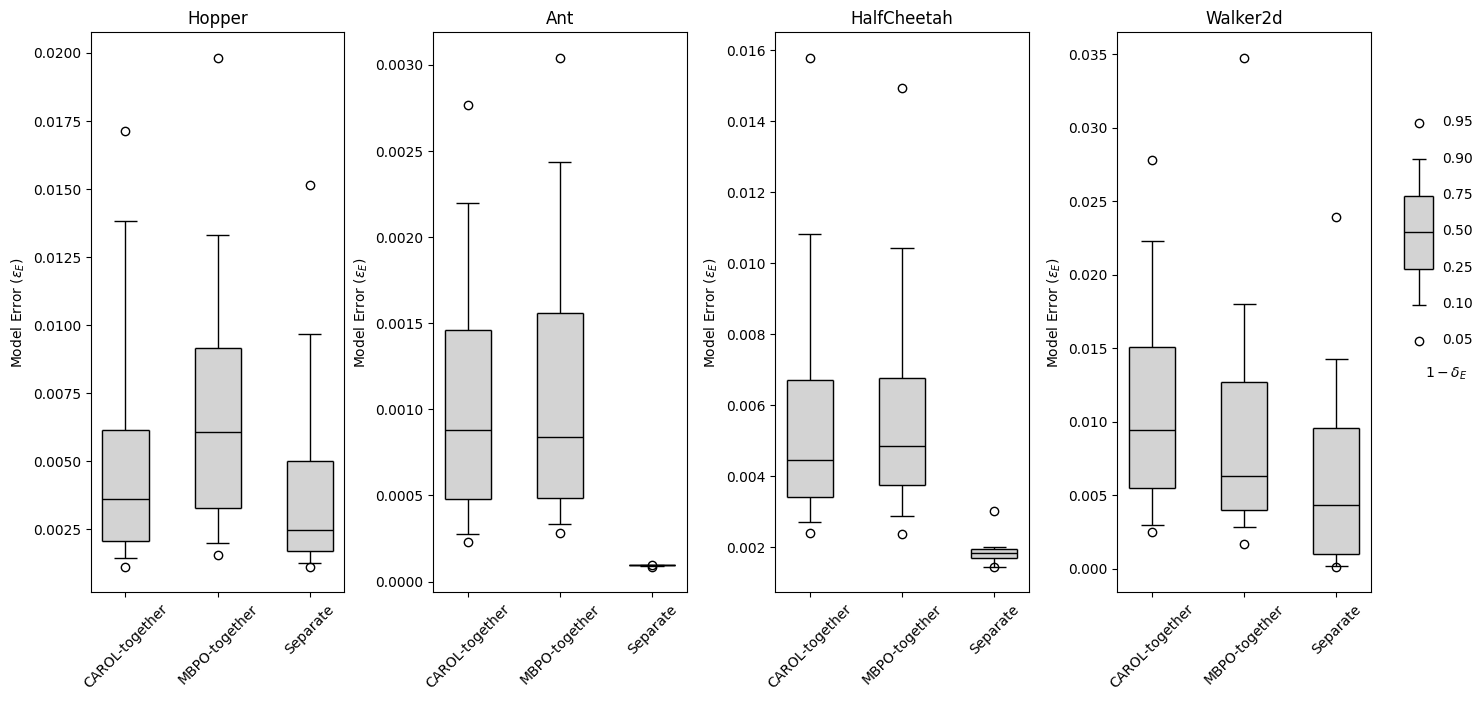}
    \caption{Demonstration of the model error distribution. \textit{CAROL-together}, \textit{MBPO-together}, and \textit{separate} represent the distribution from the models trained with CAROL, models trained with MBPO, and the models trained from datasets from rollout with random policies, respectively.}
    \label{fig:model-error}
\end{figure*}

\subsection{Certified Performance When Assuming the Model Being Perfect} \label{app:results-0error}

We demonstrate the certified performance when assuming the $\varepsilon_E$ being zero in \cref{fig:wcar-overall-0error} and \cref{fig:wcar-separate-0error}. The general trend of the certified performance does not change much, while the exact \wcar$/T$ increases. Specifically, Walker2d could give reasonable certification over longer horizons.

\begin{figure*}[ht]
    \centering
    \includegraphics[scale=0.45, bb=0 0 850 240]{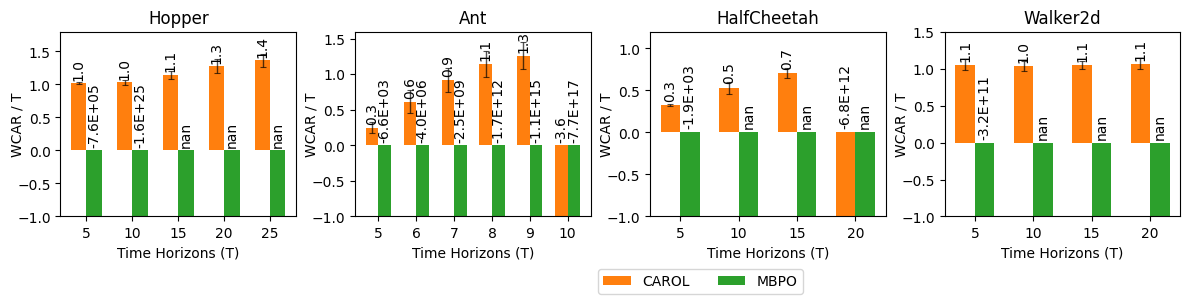}
    \caption{Certified performance of policies $\pi$ with the learned-together model, $E$. Each bar is an average of 25 starting states. The results are based on the assumption of $\varepsilon_{E}$ being $0.0$.}
    \label{fig:wcar-overall-0error}
\end{figure*}
\begin{figure*}[ht]
    \centering
    \includegraphics[scale=0.45, bb=0 0 850 240]{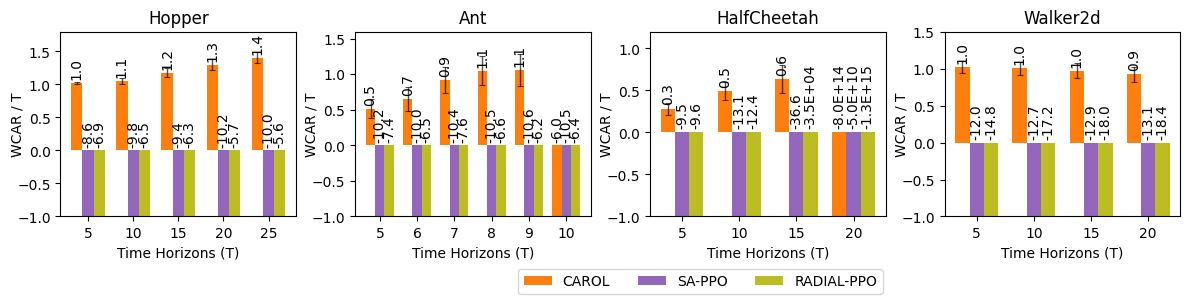}
    \caption{Certified performance of policies $\pi$ under a set of separately learned models, $\{E_i\}$. Each bar averages the learned policies on each $E_i$ of 25 starting states. The results are based on the assumption of $\varepsilon_{E}$ being $0.0$.}
    \label{fig:wcar-separate-0error}
\end{figure*}

\subsection{Certified Performance When Evaluating against Training Perturbation Range} \label{app:results-train}

We show the provability results with $\epsilon_{\text{test}}$ being $\epsilon_{\text{train}}$ in \cref{fig:wcar-overall-train} and \cref{fig:wcar-separate-train}. 

\begin{figure*}[ht]
    \centering
    \includegraphics[scale=0.45, bb=0 0 850 240]{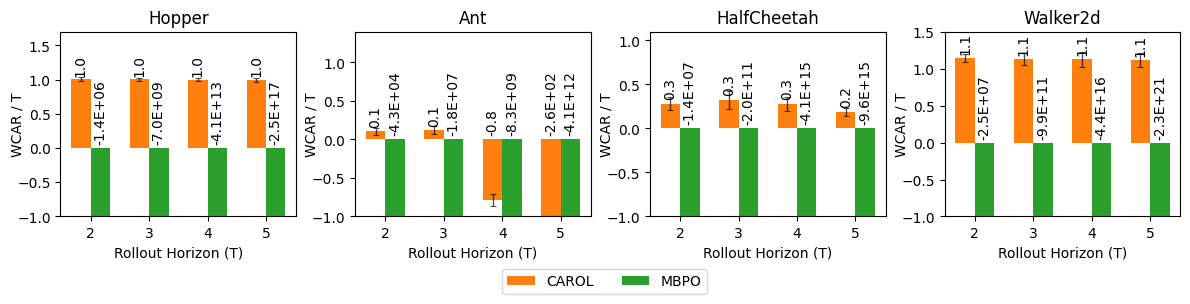}
    \caption{Certified performance of policies $\pi$ with the learned-together model, $E$. The results are based on $\varepsilon_E$ with a $1 - \delta_{E}$ of $0.9$.}
    \label{fig:wcar-overall-train}
\end{figure*}
\begin{figure*}[ht]
    \centering
    \includegraphics[scale=0.45, bb=0 0 850 240]{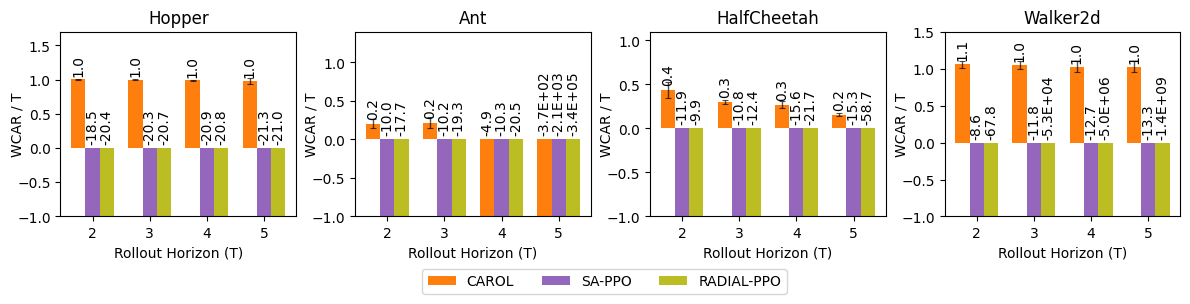}
    \caption{Certified performance of policies $\pi$ under a set of separately learned models, $\{E_i\}$. The results are based on $\varepsilon_E$ with a $1 - \delta_{E}$ of $0.9$.}
    \label{fig:wcar-separate-train}
\end{figure*}

%% file: Appendix/examples.tex
\section{Qualitative Evaluation of the Abstract Trace Lower Bound}\label{app:quali-example}

\begin{figure}
\centering
\begin{subfigure}
  \centering
  \includegraphics[scale=0.4, bb=0 0 450 350]{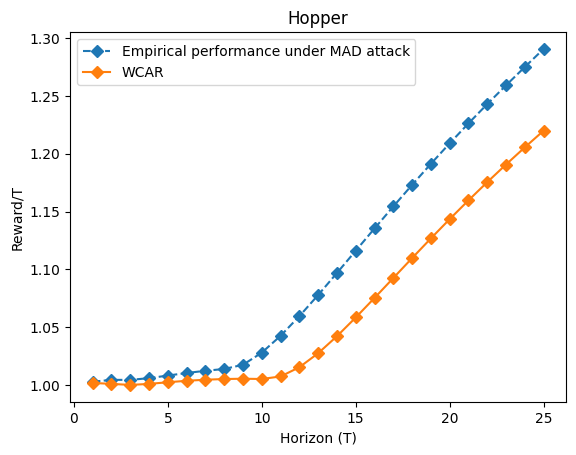}
\end{subfigure}%
\begin{subfigure}
  \centering
  \includegraphics[scale=0.4, bb=0 0 330 350]{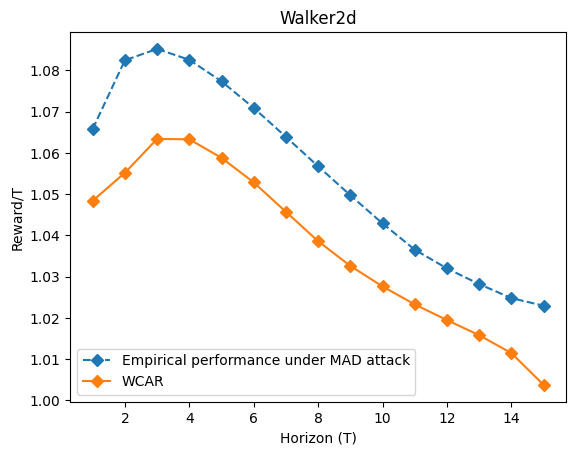}
\end{subfigure}
\caption{Examples of robustness certification of CAROL. We show the reward under one empirical attack (MAD) and WCAR (incorporating the model error with $1 - \delta_E$ being $0.90$) over horizons starting from the same initial state.}
\label{fig:certification-example}
\end{figure}

We give our attempt to evaluate and demonstrate the lower bound of the abstract traces over CAROL with examples in \Cref{fig:certification-example}. Specifically, we show the reward under one empirical attack (MAD) and our \wcar (incorporating the model error) over horizons starting from the same initial state. \wcar being always smaller than the reward under empirical attack indicates soundness. The reasonably small gap between the two lines indicates tightness. One interesting observation is that as the horizon increases, the gap increases. We give two possible explanations for this:
\begin{itemize}
\item The empirical attack is not strong enough to reveal the agents’ performance under the worst-case attack.
\item The overapproximation error and the model error from CAROL accumulate as the horizon increases.
\end{itemize}

%% file: Appendix/bound.tex
\section{Abstract Bound Propagation} \label{app:ai-bound}

Now, we give an explanation of how interval bound propagation (IBP) works. CROWN \cite{zhang2018efficient} optimizes over IBP for tighter bound (specifically for Relu and sigmoid, etc.). IBP considers the box domain in the implementation. For a program with $m$ variables, each component in the domain represents a $m$-dimensional box. Each component of the domain is a pair $b = \langle b_c, b_e \rangle$, where $b_c \in \Reals^m$ is the center of the box and $b_e \in \Reals^m_{\ge 0}$ represents the non-negative deviations. The interval concretization of the $i$-th dimension variable of $b$ is given by 
$$
[{(b_c)}_i - {(b_e)}_i, {(b_c)}_i + {(b_e)}_i].
$$

Now we give the abstract update for the box domain following \cite{mirman2018differentiable}.

\paragraph{Add.} 
For a concrete function $f$ that replaces the $i$-th element in the input vector $x \in \Reals^m$ by the sum of the $j$-th and $k$-th element:
$$
f(x) = (x_1, \dots, x_{i-1}, x_{j} + x_{k}, x_{i+1}, \dots x_m)^T.
$$
The abstraction function of $f$ is given by:
$$
f^\#(b) = \langle M \cdot b_c, M \cdot b_e \rangle,
$$ where $M \in \Reals^{m \times m}$ can replace the $i$-th element of $x$ by the sum of the $j$-th and $k$-th element by $M \cdot b_c$.

\paragraph{Multiplication.}
For a concrete function $f$ that multiplies the $i$-th element in the input vector $x \in \Reals^m$ by a constant $w$:
$$
f(x) = (x_1, \dots, x_{i-1}, w \cdot x_i, x_{i+1}, \dots, x_m)^T.
$$
The abstraction function of $f$ is given by:
$$
f^\#(b) = \langle M_w \cdot b_c, M_{|w|} \cdot b_e \rangle,
$$
where $M_w \cdot b_c$ multiplies the $i$-th element of $b_c$ by $w$ and $M_{|w|} \cdot b_e$ multiplies the $i$-th element of $b_e$ with $|w|$.

\paragraph{Matrix Multiplication.}
For a concrete function $f$ that multiplies the input $x \in \Reals^m$ by a fixed matrix $M \in \Reals^{m' \times m}$:
$$
f(x) = M \cdot x.
$$
The abstraction function of $f$ is given by:
$$
f^\#(b) = \langle M \cdot b_c, |M| \cdot b_e \rangle,
$$
where $M$ is an element-wise absolute value operation. Convolutions follow the same approach, as they are also linear operations.

\paragraph{ReLU.}
For a concrete element-wise $\textrm{ReLU}$ operation over $x \in \Reals^m$:
$$
\textrm{ReLU}(x) = (\text{max}(x_1, 0), \dots, \text{max}(x_m, 0))^T, 
$$
the abstraction function of $\textrm{ReLU}$ is given by:
$$
\textrm{ReLU}^\#(b) = \langle \frac{\textrm{ReLU}(b_c + b_e) + \textrm{ReLU}(b_c - b_e)}{2}, \frac{\textrm{ReLU}(b_c + b_e) - \textrm{ReLU}(b_c - b_e)}{2}\rangle.
$$
where $b_c + b_e$ and $b_c - b_e$ denotes the element-wise sum and element-wise subtraction between $b_c$ and $b_e$.

\paragraph{Sigmoid.}
As $\textrm{Sigmoid}$ and $\textrm{ReLU}$ are both monotonic functions, the abstraction functions follow the same approach.
For a concrete element-wise $\textrm{Sigmoid}$ operation over $x \in \Reals^m$:
$$
\textrm{Sigmoid}(x) = (\frac{1}{1 + \text{exp}(-x_1)}, \dots, \frac{1}{1 + \text{exp}(-x_m)})^T, 
$$
the abstraction function of $\textrm{Sigmoid}$ is given by:
$$
\textrm{Sigmoid}^\#(b) = \langle \frac{\textrm{Sigmoid}(b_c + b_e) + \textrm{Sigmoid}(b_c - b_e)}{2}, \frac{\textrm{Sigmoid}(b_c + b_e) - \textrm{Sigmoid}(b_c - b_e)}{2}\rangle.
$$
where $b_c + b_e$ and $b_c - b_e$ denotes the element-wise sum and element-wise subtraction between $b_c$ and $b_e$. All the above abstract updates can be easily differentiable and parallelized on the GPU.

%% file: Appendix/discussion.tex
\section{Broader Discussion about Limitations and Future Works} \label{app:discuss}
We present a detailed discussion about limitations and future directions of our work below.

\paragraph{Adversarial Setting.} We focus on the state-adversarial setting. There are broader adversarial settings related to model disturbances and model uncertainty \cite{mankowitz2019robust}. Exploring the certified learning over environment dynamics perturbation is of interest. Specifically, we do not require a predefined model. We learn a model where the model misspecification amounts to supervised learning error. Future works would incorporate the potential disturbance of the environment in training to learn the dynamics model under \sys framework.

\paragraph{Dimensionality.} We focus on control benchmarks in this work. Related works about certification \cite{wu2021crop} use higher dimensional environments (e.g., Atari). However, the methods in CROP \cite{wu2021crop} primarily work on discrete state/action space and assume a deterministic environment. \sys is more general; while CROP does post-hoc verification, we focus on certified learning, where verification is integrated with learning. In addition, to the best of our knowledge, the environments we evaluate over have the largest dimensionality in certified RL papers \cite{dawson2022safe, chang2021stabilizing}.

\paragraph{Stochasticity.} We have challenges in handling highly random environments, which is a fundamental limitation of all certified learning techniques. We consider the stochasticity in the theoretical analysis by incorporating the variance of $R^{\#}$in the soundness bound. When stochasticity is large, the tightness of our bound may be affected.

\paragraph{Complexity.} Certified learning is more expensive than regular learning due to the requirement of certifiability and soundness guarantee. In training, the symbolic state is represented by the center and the width of a box (2x information representation per state). Propagating over a box needs an additional 2x computation compared to computation without symbolic states. Consequently, \sys is approximately two times slower than the base RL algorithm per step.

In summary, \sys’s performance is limited in the very large-scale MDPs over long rollout horizons mainly due to two reasons:
\begin{itemize}
\item Efficient abstract interpretation domains (e.g., Interval/Box) can give accumulated over-approximation error.
\item Tighter abstract interpretation domains are expensive in training and may not be differentiable. (e.g. bounded Zonotopes \cite{scott2016constrained} or Polyhedra \cite{singh2017fast}).
\end{itemize}

We believe that future works on tighter, more efficient, and differentiable abstract interpretation techniques would benefit \sys as our framework is not built on top of one particular abstract interpretation method. Additionally, a broader setting of adversarial perturbations would be an interesting future direction to extend our work.